\documentclass[conference, 10pt]{IEEEtran}
\ifCLASSINFOpdf
\else
\fi
\usepackage{graphicx}
\usepackage{amsmath}
\usepackage{amsfonts}
\usepackage{amssymb}
\usepackage[amssymb]{SIunits}
\usepackage{longtable}

\usepackage{import}
\usepackage{color}
\usepackage{algorithm,algpseudocode}
\usepackage{amsthm}
\usepackage{commath}
\usepackage{subcaption} 
\usepackage{lipsum}
\usepackage{multicol}
\usepackage{multirow}

\newtheorem{myremark}{Remark}
\newtheorem{mytheo}{Theorem}

\newtheorem{myprop}{Proposition}
\newtheorem{myproperty}{Property}
\newtheorem{myassump}{Assumption}

\begin{document}
%
\title{R3Net: Random Weights, Rectifier Linear Units and
	Robustness for Artificial Neural Network}

\author{\IEEEauthorblockN{Arun Venkitaraman, Alireza M. Javid, and Saikat Chatterjee}
\IEEEauthorblockA{Department of Information Science and Engineering,\\School of Electrical Engineering and Computer Science\\
KTH Royal Institute of Technology, Stockholm, Sweden--10044\\
Email: \{arunv,almj,sach\}@kth.se}
}


%


\maketitle

\begin{abstract}
We consider a neural network architecture with randomized features, a sign-splitter, followed by rectified linear units (ReLU). We prove that our architecture exhibits robustness to the input perturbation: the output feature of the neural network exhibits a Lipschitz continuity in terms of the input perturbation. We further show that the network output exhibits a  discrimination ability that inputs that are not arbitrarily close generate output vectors which maintain distance between each other obeying a certain lower bound. This ensures that two different inputs remain discriminable while contracting the distance in the output feature space.
\end{abstract}


%
\IEEEpeerreviewmaketitle

\section{Introduction}
Neural networks and deep learning architectures have revolutionized data analysis over the last decade \cite{DL_SPMag}. Appropriately trained neural networks have been shown to excel in classification and regression tasks, in many cases outperforming humans \cite{Russakovsky2015,DodgeK17b}. The field is continually being enriched with active research pushing classification performance to increasingly higher levels. The rapidly increasing computational power and data storage have only added to the power of neural networks. However, very little is known regarding why the networks are able to gain this superior performance. In addition, it is known that learnt neural networks can be fragile when it comes to handling perturbations in the input \cite{genadvnet_SPMag,genadvnet_1,genadvnet_2,genadvnet_3,genadvnet_4}. It is hypothesized that this is because of the layers of the network being trained to fit the data closely. For example, in image classification, additive noise at very low signal-to-noise ratio levels added to images have been known to disporportionately change the class labels, even when the additive noise is practically unnoticeable to human eyes \cite{genadvnet_1}. Such instability makes the network easy target to adversarial attacks and hacking \cite{genadvnet_3}. This observation has led many researchers to investigate and develop deep networks with features that exhibit robustness to deformation or perturbation by building on invariances \cite{deepscattering_1,Bruna6522407,Mallat20150203,Wiatowski_8051085,WiatowskiB15a}. 

Randomness of features has been used with great success as a mean of reducing the computational complexity of neural networks while achieving comparable performance as with fully learnt networks \cite{mathematicsDN,giryes_randomweights,proglearnnet_17}. In the case of the simple, yet effective, extreme learning machine (ELM), all layers of the network are assigned randomly chosen weights and the learning takes place only at the extreme layer \cite{elm_Huang2012,elm_HUANG2015,elm_8085130}. It has also been shown recently that a performance similar to fully learnt networks may be achieved by training a network with most of the weights assigned randomly and only a small fraction of them being updated throughout the layers \cite{randomDN}. These approaches indicate that randomness has much potential in terms of high-performance at low computational complexity. This motivates us to propose a neural network architecture that uses random weights in the layers followed by a structured sign splitter and rectified linear unit (ReLU) activation functions. We show that the output of each layer exhibits robustness in terms of perturbation in the input--the perturbation of the output of each layer has an upper and a lower bound in terms of the input perturbation. We believe that this is a step towards mathematically explaining the efficiency of random weights and ReLUs in neural networks observed in practice. We name our proposed architecture as R3Net, motivated by words `random weights', `rectifier linear units', and `robustness'. In this article, we show only analytical results, and refrain from providing simulation results. Simulation results will be shown in an extended manuscript later.

\subsection{Notation}
For a scalar $x \in \mathbb{R}$, we denote its sign as $s(x) \triangleq \mathrm{sign}(x)$ and magnitude as $|x|$. Then, we have $x = s(x) \odot |x|$. Sign takes values in the set $\{+1,0,-1\}$. For a vector $\mathbf{x}$, the corresponding sign vector is found by component wise operation and the sign vector is denoted by $\mathbf{s}(\mathbf{x})$. We use $|\mathbf{x}|$ to denote the magnitude of a real vector $\mathbf{x}$ where magnitude is used scalar-wise. For vector $\mathbf{x}$, we denote the non-negative part by $\mathbf{x}^{+}$ and non-positive part by $\mathbf{x}^{-}$, such that $\mathbf{x} = \mathbf{x}^{+} + \mathbf{x}^{-}$.  We denote the ReLU function by $g(\cdot)$ such that $g(x)= \mathrm{max}(0,x)$. We then denote by $\mathbf{g}(\mathbf{x})$ the stack of ReLU activation functions applied component-wise on $\mathbf{x}$. Therefore, $\mathbf{g}(\mathbf{x}) = \mathbf{x}^{+}$. We use $\mathcal{M}$ to denote a set and $\mathcal{M}^c$ to its complement set. Cardinality of a set $\mathcal{M}$ is denoted by $|\mathcal{M}|$. We use $\|\cdot\|$ to denote the $\ell_2$ norm of a vector, and $\|\cdot\|_F$ to denote the Frobenius norm.

\section{Noise robustness and discrimination ability}
It is well known that ANNs which involve a chain of blocks comprised of linear  and nonlinear transformations lead to impressive performance given large amounts of reliable training data. However, as shown recently, this is not sufficient to guarantee that the ANN is a stable one \cite{genadvnet_SPMag,genadvnet_1,genadvnet_2,genadvnet_3,genadvnet_4}. 
In order that an ANN be stable, it is desirable that it possesses 
{\em noise robustness}. Let $\mathbf{x}_1$ and $\mathbf{x}_2$ be two input vectors such that $\mathbf{x}_1 \neq \mathbf{x}_2$,  and the corresponding {\color{black}feature vectors generated by ANN be $\tilde{\mathbf{y}}_1 = \mathbf{f}(\mathbf{x}_1)$ and $\tilde{\mathbf{y}}_2 = \mathbf{f}(\mathbf{x}_2)$. In order to characterize a scenario with input perturbation noise $\Delta$, we assume $\mathbf{x}_2 = \mathbf{x}_1 + \Delta$. Then, the desired property of ANN in terms of robustness is expressible as
\begin{align}
	 \|  \tilde{\mathbf{y}}_1 - \tilde{\mathbf{y}}_2 \|^2 = \|  \mathbf{f}(\mathbf{x}_1) - \mathbf{f}(\mathbf{x}_1) \|^2  \leq B \|  \mathbf{x}_1 - \mathbf{x}_2 \|^2,
\end{align}
where $0 < B \leq 1$.
Further, it is often desirable that the feature vector continues to maintain a certain minimum distance between if the input vectors are different. In other words, we would like to have the following property:
\begin{align}
	A  \|  \mathbf{x}_1 - \mathbf{x}_2 \|^2 \leq
	\|  \tilde{\mathbf{y}}_1 - \tilde{\mathbf{y}}_2 \|^2,
\end{align}
where $0<A\leq B$.
This ensures that the targets do not go arbitrarily close when the inputs are not close and it is possible to discriminate one feature from the other. The upper bound helps to provide noise robustness: the perturbation in the output is a constant multiple of input perturbation $\Delta$.}

\section{Single Block Construction}

In order to investigate the desired properties, we first consider a single block of ANN, usually referred to as a layer in neural network literature. The block has an input vector $\mathbf{q} \in \mathbb{R}^{q\times 1}$ and an output vector $\mathbf{y} = \mathbf{g}(\mathbf{W}\mathbf{q})$, where $\mathbf{W}$ is the linear transform or weight matrix and $\mathbf{g}(\cdot)$ the component-wise nonlinearity (the ReLU in our case). The dimension of $\mathbf{y}$ is the number of neurons in the block. If we can ensure that one block of the ANN provides both noise robustness and point discrimination property, then, the full ANN comprising multiple blocks connected sequentially can be guaranteed to hold robustness and discriminative properties.   
This argument boils down to the construction of matrix $\mathbf{W}$ which promotes noise robustness and discriminative power in each block. 

\subsection{ReLU function: Properties and a limitation}
We now discuss some properties and a limitation of the ReLU. 
The ReLU operation on a scalar $x$ is given by
\begin{align*}
g(x)\triangleq\max(x,0).
\end{align*}
As a consequence of which we can see that the vector transformation $\mathbf{g}(\cdot)$ consisting of component-wise ReLU operations has 
%

\begin{myproperty}
	ReLU function provides sparse output vector $\mathbf{y}$ such that $\| \mathbf{y} \|_0 \leq \mathrm{dim}(\mathbf{y})$.
\end{myproperty}
\begin{myproperty}
	\label{property:property2}
	Let us consider $\mathbf{z} = \mathbf{W} \mathbf{q}$. For two vectors $\mathbf{q}_1$ and $\mathbf{q}_2$, we have corresponding vectors $\mathbf{z}_1 = \mathbf{W} \mathbf{q}_1$ and $\mathbf{z}_2 = \mathbf{W} \mathbf{q}_2$, and output vectors $\mathbf{y}_1 = \mathbf{g}(\mathbf{z}_1) = \mathbf{g}(\mathbf{W}\mathbf{q}_1)$ and $\mathbf{y}_2 = \mathbf{g}(\mathbf{z}_2) = \mathbf{g}(\mathbf{W}\mathbf{q}_2)$. 
	Then, we have the following relation 
	\begin{eqnarray}
		0 \leq \| \mathbf{y}_1 - \mathbf{y}_2 \|^2 = \| \mathbf{g}(\mathbf{z}_1) - \mathbf{g}(\mathbf{z}_2) \|^2 \leq \| \mathbf{z}_1 - \mathbf{z}_2 \|^2.
	\end{eqnarray}
		\begin{proof}
		For scalars $x_1$ and $x_2$, we have $y_1 = g(x_1)$ and  $y_2 = g(x_2)$. We have the following relation
		\begin{eqnarray*}
			(y_1 - y_2)^2 = \left \{
			\begin{array}{lr}
				(x_1 - x_2)^2 	& \mathrm{if} \,\, x_1 > 0, x_2 > 0  \\
				x_1^2 			& \mathrm{if} \,\, x_1 > 0, x_2 < 0  \\
				x_2^2 			& \mathrm{if} \,\, x_1 < 0, x_2 > 0  \\
				0				& \mathrm{if} \,\, x_1 < 0, x_2 < 0
			\end{array}
			\right.
			.
		\end{eqnarray*}
		Therefore, we observe that the ReLU function satisfies
		\begin{eqnarray*}
			0 \leq (y_1 - y_2)^2 \leq (x_1 - x_2)^2.
		\end{eqnarray*}
		Then, considering the vectors $\mathbf{y}_1 = \mathbf{g}(\mathbf{z}_1) = \mathbf{g}(\mathbf{W}\mathbf{q}_1)$ and $\mathbf{y}_2 = \mathbf{g}(\mathbf{z}_2) = \mathbf{g}(\mathbf{W}\mathbf{q}_2)$, we have that
		\begin{align*}
		0 &\leq \| \mathbf{y}_1 - \mathbf{y}_2 \|^2 = \sum_i (y_1(i) - y_2(i))^2 \leq  \sum_i (z_1(i) - z_2(i))^2  \nonumber\\&= \| \mathbf{z}_1 - \mathbf{z}_2 \|^2
		\end{align*}
		where $y_1(i)$ is the $i$'th scalar element of $\mathbf{y}_1$ and $z_1(i)$ is the the $i$'th scalar element of $\mathbf{z}_1$.		
	\end{proof}
\end{myproperty}
\noindent The upper bound relation in Property 3 implies that ReLU is Lipschitz continuous with Lipschitz constant 1. This show that the output perturbation of the ReLU is bounded by the input perturbation thereby providing noise robustness. On the other hand, the lower bound being zero does not support our need of maintaining a minimum distance between two points $\mathbf{y}_1$ and $\mathbf{y}_2$. 
An example of the extreme effect is the case when $\mathbf{z}_1$ and $\mathbf{z}_2$ are non-positive vectors, and we get $\| \mathbf{y}_1 - \mathbf{y}_2 \|^2 = 0$. 
This is then a limitation of the ReLU in achieving a good discriminative power. 

\subsection{Overcoming the limitation}
We now engineer a remedy of the limitation of the ReLU. Let us consider $\mathbf{z} = \mathbf{W} \mathbf{q} \in \mathbb{R}^{n}$ and $\bar{\mathbf{y}}=\mathbf{g}(\mathbf{V}\mathbf{z})$ where $\mathbf{V}$ is a linear transform matrix. In other words, we introduce an additional linear transform after $\mathbf{W}$ in the block.
For two vectors $\mathbf{q}_1$ and $\mathbf{q}_2$, we have the corresponding vectors $\mathbf{z}_1 = \mathbf{W} \mathbf{q}_1$ and $\mathbf{z}_2 = \mathbf{W} \mathbf{q}_2$, and output vectors $\bar{\mathbf{y}}_1 = \mathbf{g}(\mathbf{V}\mathbf{z}_1)$ and $\bar{\mathbf{y}}_2 = \mathbf{g}(\mathbf{V}\mathbf{z}_2)$. Our interest is to show that there exists a $\mathbf{V}$ matrix for which we have both noise robustness and discriminative power properties, given $\mathbf{W}$ and the ReLU.

\begin{myprop}
	\label{proposition:proposition_with_Vn}
	Let us construct a $\mathbf{V}$ matrix as follows
	\begin{eqnarray}
		\mathbf{V} = \left[ 
		\begin{array}{c}
			\mathbf{I}_n \\
			- \mathbf{I}_n
		\end{array}
		\right] \triangleq \mathbf{V}_n.
	\end{eqnarray}
	For the output vectors $\bar{\mathbf{y}}_1 = \mathbf{g}(\mathbf{V}_n\mathbf{z}_1) \in \mathbb{R}^{2n}$ and $\bar{\mathbf{y}}_2 = \mathbf{g}(\mathbf{V}_n\mathbf{z}_2) \in \mathbb{R}^{2n}$, we have
	\begin{eqnarray}
		0 < \kappa \| \mathbf{z}_1 - \mathbf{z}_2 \|^2 \leq \| \bar{\mathbf{y}}_1 - \bar{\mathbf{y}}_2 \|^2  \leq \| \mathbf{z}_1 - \mathbf{z}_2 \|^2,
		\label{eq:LBAndUB_1}
	\end{eqnarray}
	where $0<\kappa \leq 1$ and $\kappa$ is a function of $\mathbf{z}_1$ and $\mathbf{z}_2$.
\end{myprop}
\begin{proof}
We have $\mathbf{z} = \mathbf{W} \mathbf{q} \in \mathbb{R}^{n}$ and $\bar{\mathbf{y}}=\mathbf{g}(\mathbf{V}_n\mathbf{z}) \in \mathbb{R}^{2n}$ where 
\begin{eqnarray*}
	\mathbf{V}_n = \left[ 
	\begin{array}{c}
		\mathbf{I}_n \\
		- \mathbf{I}_n
	\end{array}
	\right].
\end{eqnarray*}
For two vectors $\mathbf{q}_1$ and $\mathbf{q}_2$, we have corresponding vectors $\mathbf{z}_1 = \mathbf{W} \mathbf{q}_1$ and $\mathbf{z}_2 = \mathbf{W} \mathbf{q}_2$, and output vectors $\bar{\mathbf{y}}_1 = \mathbf{g}(\mathbf{V}_n\mathbf{z}_1)$ and $\bar{\mathbf{y}}_2 = \mathbf{g}(\mathbf{V}_n\mathbf{z}_2)$.  Let us define a set 
\begin{eqnarray}
	\mathcal{M}(\mathbf{z}_1,\mathbf{z_2}) = \{ i | s(z_1(i)) = s(z_2(i)) \neq 0 \} \subseteq \{1,2, \hdots , n\}.
\end{eqnarray}
Then, we have
\begin{align}
		\| \mathbf{z}_1 - \mathbf{z}_2 \|^2 & = \displaystyle \sum_{i=1} ( z_1(i) - z_2(i) )^2 \nonumber\\
		& = \displaystyle\sum_{i} (  s(z_1(i)) \odot |z_1(i)|    -  s(z_2(i)) \odot |z_2(i)|  ) )^2 \nonumber\\
		& = \displaystyle \sum_{i \in \mathcal{M}(\mathbf{z}_1,\mathbf{z}_2)} ( |z_1(i)| - |z_2(i)| )^2 \nonumber\\
		&\quad\displaystyle+ \sum_{i \in \mathcal{M}^c(\mathbf{z}_1,\mathbf{z}_2)} ( |z_1(i)| + |z_2(i)| )^2.
	\label{eq:z_distance}
\end{align}
Expressing the vectors in terms of $\mathbf{z} = \mathbf{z}^{+} + \mathbf{z}^{-} = (\mathbf{s}(\mathbf{z}^{+}) \odot |\mathbf{z}^{+}| ) + (\mathbf{s}(\mathbf{z}^{-}) \odot |\mathbf{z}^{-}| )$, we have the outputs of the ReLU operation as follows
\begin{eqnarray*}
	\bar{\mathbf{y}}_1 = \mathbf{g}(\mathbf{V}_n \mathbf{z}_1) = \left[
	\begin{array}{c}
		|\mathbf{z}_1^{+}|  \\
		|\mathbf{z}_1^{-}| 
	\end{array}
	\right] \,\, \mathrm{and} \,\,
	\bar{\mathbf{y}}_2 = \mathbf{g}(\mathbf{V}_n \mathbf{z}_2) = \left[
	\begin{array}{c}
		|\mathbf{z}_2^{+}|  \\
		|\mathbf{z}_2^{-}| 
	\end{array}
	\right].
\end{eqnarray*}

\begin{figure*}[t]
	\begin{align}
	\| \bar{\mathbf{y}}_1 - \bar{\mathbf{y}}_2 \|^2  &= \| |\mathbf{z}_1^{+}| - |\mathbf{z}_2^{+}| \|^2 + \| |\mathbf{z}_1^{-}| - |\mathbf{z}_2^{-}| ) \|^2 \nonumber\\
	&= \displaystyle \sum_{i \in \mathcal{M}(|\mathbf{z}_1^{+}| ,|\mathbf{z}_2^{+}| )} ( |z_1^{+}(i)| - |z_2^{+}(i)| )^2 \quad+\displaystyle  \sum_{i \in \mathcal{M}^c(|\mathbf{z}_1^{+}| ,|\mathbf{z}_2^{+}| )} ( |z_1^{+}(i)| + |z_2^{+}(i)| )^2 \nonumber\\&\quad+ \displaystyle \sum_{i \in \mathcal{M}(|\mathbf{z}_1^{-}| ,|\mathbf{z}_2^{-}| )} ( |z_1^{-}(i)| - |z_2^{-}(i)| )^2 \quad+\displaystyle  \sum_{i \in \mathcal{M}^c(|\mathbf{z}_1^{-}| ,|\mathbf{z}_2^{-}| )} ( |z_1^{-}(i)| + |z_2^{-}(i)| )^2 \nonumber\\
	&= \displaystyle \sum_{i \in \mathcal{M}(|\mathbf{z}_1^{+}| ,|\mathbf{z}_2^{+}| )} ( |z_1^{+}(i)| - |z_2^{+}(i)| )^2  +\displaystyle  \displaystyle \sum_{i \in \mathcal{M}(|\mathbf{z}_1^{-}| ,|\mathbf{z}_2^{-}| )} ( |z_1^{-}(i)| - |z_2^{-}(i)| )^2 \nonumber\\&\quad+ \displaystyle \sum_{i \in \mathcal{M}^c(|\mathbf{z}_1^{+}| ,|\mathbf{z}_2^{+}| )} ( |z_1^{+}(i)| + |z_2^{+}(i)| )^2  +\displaystyle  \sum_{i \in \mathcal{M}^c(|\mathbf{z}_1^{-}| ,|\mathbf{z}_2^{-}| )} ( |z_1^{-}(i)| + |z_2^{-}(i)| )^2 \nonumber\\
	& = \displaystyle \sum_{i \in \mathcal{M}(\mathbf{z}_1,\mathbf{z}_2)} ( |z_1(i)| - |z_2(i)| )^2 +\displaystyle  \sum_{i \in \mathcal{M}^c(\mathbf{z}_1,\mathbf{z}_2)}  |z_1(i)|^2 + |z_2(i)|^2.
	\label{eq:y_bar_distance}
	\end{align}
\end{figure*}
From \eqref{eq:z_distance} and \eqref{eq:y_bar_distance}, we have the following relation
\begin{eqnarray*}
	\| \bar{\mathbf{y}}_1 - \bar{\mathbf{y}}_2 \|^2 \leq \| \mathbf{z}_1 - \mathbf{z}_2 \|^2,
\end{eqnarray*}
where equality holds when $\mathcal{M}^c = \emptyset$, which is the case when sign patterns of $\mathbf{z}_1$ and $\mathbf{z}_2$ match exactly. We next define the parameter
\begin{eqnarray}
	\gamma \triangleq \displaystyle\max_{\mathbf{z}_1,\mathbf{z}_2}\Big( & \sum_{i \in \mathcal{M}(\mathbf{z}_1,\mathbf{z}_2)} & ( |z_1(i)| - |z_2(i)| )^2,\nonumber\\ & \sum_{i \in \mathcal{M}^c(\mathbf{z}_1,\mathbf{z}_2)} &  |z_1(i)|^2 + |z_2(i)|^2 \Big).\nonumber
\end{eqnarray} 
We note that $0 < \gamma \leq \| \bar{\mathbf{y}}_1 - \bar{\mathbf{y}}_2 \|^2 \leq \| \mathbf{z}_1 - \mathbf{z}_2 \|^2$ and hence, it follows that
\begin{eqnarray}
	0 < \frac{\gamma}{\| \mathbf{z}_1 - \mathbf{z}_2 \|^2} \| \mathbf{z}_1 - \mathbf{z}_2 \|^2 \leq \| \bar{\mathbf{y}}_1 - \bar{\mathbf{y}}_2 \|^2 \leq \| \mathbf{z}_1 - \mathbf{z}_2 \|^2.\nonumber
\end{eqnarray}
On defining
\begin{eqnarray}
	\kappa \triangleq \frac{\gamma}{\| \mathbf{z}_1 - \mathbf{z}_2 \|^2},
\end{eqnarray} we get \eqref{eq:LBAndUB_1}
since $0 < \kappa \leq 1$.
\end{proof}

The difference signal $\Delta \mathbf{z} = \mathbf{z}_1 - \mathbf{z}_2$ can be treated as the perturbation noise. Note that $\Delta \mathbf{z} = \mathbf{z}_1 - \mathbf{z}_2 = \mathbf{W}[\mathbf{q}_1 - \mathbf{q}_2] = \mathbf{W} \, \Delta \mathbf{q}$. To investigate the effect of the perturbation noise, we now state our main assumption.

\begin{myassump}
	\label{assump:sign_pattern_change_in_z}
	A $\Delta \mathbf{z}$ with a low strength (that means $\| \Delta \mathbf{z} \|^2$ is low) does not create a high change in the sign patterns of $\mathbf{z}_1$ and $\mathbf{z}_2$. This means that for a small perturbation, $\mathcal{M}(\mathbf{z}_1,\mathbf{z_2}) = \{ i | s(z_1(i)) = s(z_2(i)) \neq 0 \}$ is close to the entire index set and $\mathcal{M}^c(\mathbf{z}_1,\mathbf{z_2})$ is close to an empty set. On the other hand, for a high perturbation noise strength, we assume that $\mathcal{M}(\mathbf{z}_1,\mathbf{z_2}) = \{ i | s(z_1(i)) = s(z_2(i)) \neq 0 \}$ is close to an empty set and $\mathcal{M}^c(\mathbf{z}_1,\mathbf{z_2})$ is close to a full set.
\end{myassump}

\begin{myremark}[Tightness of bounds and effect of noise]
	\label{remark:noise_effect_in_z}
	For a low perturbation noise strength, we have $\| \bar{\mathbf{y}}_1 - \bar{\mathbf{y}}_2 \|^2 \approx \| \mathbf{z}_1 - \mathbf{z}_2 \|^2$ and $\kappa \approx 1$. This follows from the proof of Proposition~\ref{proposition:proposition_with_Vn}, specifically equations \eqref{eq:z_distance} and \eqref{eq:y_bar_distance}. In fact, if $\mathcal{M}^c = \emptyset$ then $\| \bar{\mathbf{y}}_1 - \bar{\mathbf{y}}_2 \|^2 = \| \mathbf{z}_1 - \mathbf{z}_2 \|^2$ and $\kappa = 1$. We interpret that a low perturbation noise passes through the transfer function $\mathbf{g}(\mathbf{V}\mathbf{z})$ almost unhindered. On the other hand, a perturbation noise with high strength is attenuated. Let us construct an illustrative example. Assume that $\mathcal{M}^c(\mathbf{z}_1,\mathbf{z_2})$ is a full set and $\forall i \in \mathcal{M}^c, \,\, |z_1(i)| = |z_2(i)|$. In that case $\| \bar{\mathbf{y}}_1 - \bar{\mathbf{y}}_2 \|^2 = 0.5 \| \mathbf{z}_1 - \mathbf{z}_2 \|^2$ and we can comment that the perturbation noise is significantly attenuated.
\end{myremark}

\begin{myproperty}
The output vector $\bar{\mathbf{y}} = \mathbf{g}(\mathbf{V}_n \mathbf{z}) \in \mathbb{R}^{2n}$ is sparse and $\| \bar{\mathbf{y}} \|_0 \leq \frac{1}{2} \times \mathrm{dim}(\bar{\mathbf{y}}) = n$. 
\end{myproperty}
\begin{proof}
	 Let us assume that $\mathbf{z}$ has no scalar component that is zero. Now, for an extreme case where $\mathbf{z}$ is positive, we have $\mathbf{z}^{-} = \mathbf{0}$. Similarly, if $\mathbf{z}$ is negative, then, we have $\mathbf{z}^{+} = \mathbf{0}$. For these two extreme cases $\| \bar{\mathbf{y}} \|_0 = n$. In any other case when $\mathbf{z}$ has zero scalars, the inequality result will follow.
\end{proof}

\subsection{Input-output relation}
We now establish relation between the block input vector $\mathbf{q} \in \mathbb{R}^m$ and output vector $\bar{\mathbf{y}} \in \mathbb{R}^{2n}$. For two vectors $\mathbf{q}_1$ and $\mathbf{q}_2$, we have corresponding vectors $\mathbf{z}_1 = \mathbf{W} \mathbf{q}_1$ and $\mathbf{z}_2 = \mathbf{W} \mathbf{q}_2$, and output vectors $\bar{\mathbf{y}}_1 = \mathbf{g}(\mathbf{V}_n\mathbf{z}_1) = \mathbf{g}(\mathbf{V}_n \mathbf{W}\mathbf{q}_1) $ and $\bar{\mathbf{y}}_2 = \mathbf{g}(\mathbf{V}_n\mathbf{z}_2) = \mathbf{g}(\mathbf{V}_n \mathbf{W}\mathbf{q}_2)$. Our interest is to show that it is possible to construct a $\mathbf{W} \in \mathbb{R}^{n \times m}$ matrix for which we have both noise robustness and discriminative power properties.

\begin{myassump}
	We assume that the input vector $\mathbf{q}$ is a sparse vector, that means the sparsity level $\nu \triangleq \| \mathbf{q} \|_0  \leq \mathrm{dim}(\mathbf{q}) \triangleq m$
\end{myassump}
\noindent The assumption is valid if the vector $\mathbf{q}$ is considered as the output of a similar block in a feedforward network.

\begin{myassump}
	From theory of compressed sensing, specifically restricted-isometry-property (RIP) of random matrices, we assume that if $n=O(\nu \log m)$, then, we can construct $\mathbf{W}$ matrix with a restricted-isometry-constant (RIC) $\delta \in (0,1)$, such that the following result holds
	\begin{eqnarray}
		(1-\delta) \| \mathbf{q}_1 - \mathbf{q}_2 \|^2 \leq \| \mathbf{z}_1 - \mathbf{z}_2 \|^2 \leq (1+\delta) \| \mathbf{q}_1 - \mathbf{q}_2 \|^2,
		\label{eq:RIP_property_necessary}
	\end{eqnarray}
	where sparse vectors $\mathbf{q}_1$ and $\mathbf{q}_2$ have a sparsity level $\nu$.
\end{myassump}

\begin{myremark}[Construction of $\mathbf{W}$ matrix]
	The $\mathbf{W} = \{ w_{ij} \}$ matrix is a randomly drawn instance. A popular approach is to draw $w_{ij}$ independently from the Gaussian distribution $\mathcal{N}(0,\frac{1}{n})$. One may also use other distributions, such as Bernoulli, Rademacher \cite{Baranuik_CS}.
\end{myremark}

\begin{myprop}
	For a randomly constructed $\mathbf{W} \in \mathbb{R}^{n \times m}$ matrix with number of rows $n=O(\| \mathbf{q} \|_0 \log m)$, we can combine the inequalities in \eqref{eq:LBAndUB_1} and \eqref{eq:RIP_property_necessary} to get
	\begin{eqnarray}
		\kappa(1-\delta) \| \mathbf{q}_1 - \mathbf{q}_2 \|^2 \leq \| \bar{\mathbf{y}}_1 - \bar{\mathbf{y}}_2 \|^2 \leq (1+\delta) \| \mathbf{q}_1 - \mathbf{q}_2 \|^2.
	\end{eqnarray}
	If we construct a block with the transfer function $\bar{\mathbf{y}} = \mathbf{g}(\mathbf{V}_n\mathbf{z}) = \mathbf{g}(\mathbf{V}_n \mathbf{W}\mathbf{q})$ where we use randomly chosen $\mathbf{W}$ matrix with appropriate size, then, the block provides noise robustness and discriminative power properties.
	
	When there is no requirement on $\mathbf{q}$ to be sparse, we can construct $\mathbf{W}  \in \mathbb{R}^{n \times m}$ as an instance of random orthonormal matrix, such that $n \geq m$ and $\mathbf{W}^{\top}\mathbf{W}=\mathbf{I}_m$. In that case, we have the relation 
	\begin{eqnarray}
		\| \mathbf{q}_1 - \mathbf{q}_2 \|^2 = \| \mathbf{z}_1 - \mathbf{z}_2 \|^2
	\end{eqnarray}
	for a pair of $(\mathbf{q}_1,\mathbf{q}_2)$ irrespective of sparsity. Combining the above relation with the relation \eqref{eq:LBAndUB_1}, we have the following result when we use random instance of orthonormal $\mathbf{W}$ matrix
	\begin{eqnarray}
		\kappa \| \mathbf{q}_1 - \mathbf{q}_2 \|^2 \leq \| \bar{\mathbf{y}}_1 - \bar{\mathbf{y}}_2 \|^2 \leq  \| \mathbf{q}_1 - \mathbf{q}_2 \|^2.
	\end{eqnarray}
\end{myprop} 

\begin{myremark}[Tightness of bounds and effect of noise]
	\label{remark:noise_effect_in_q}
	With the relation $\Delta \mathbf{z} = \mathbf{W} \, \Delta \mathbf{q}$, we assume that Assumption~\ref{assump:sign_pattern_change_in_z} holds as the perturbation noise $\Delta \mathbf{q}$ varies. Therefore, we follow similar arguments in Remark~\ref{remark:noise_effect_in_z}.
	For a low perturbation noise strength $\| \Delta \mathbf{q} \|^2$, we have $\| \bar{\mathbf{y}}_1 - \bar{\mathbf{y}}_2 \|^2 \approx \| \mathbf{q}_1 - \mathbf{q}_2 \|^2$ and $\kappa \approx 1$. We interpret that a low perturbation noise passes through the transfer function $\mathbf{g}(\mathbf{V} \mathbf{W} \mathbf{q})$ almost unhindered. On the other hand, a perturbation noise with high strength is attenuated. 
\end{myremark}

\section{Block Chain Construction}
A feedforward ANN is comprised of similar operational blocks in a chain. Let us consider two blocks in feedforward connection. These can be $l$'th and $(l+1)$'th blocks of an ANN. For the $l$'th block, we use a superscript $(l)$ to denote appropriate variables and systems. Let the $l$'th block have $m^{(l)}$ nodes. The input to the $l$'th block $\mathbf{q}^{(l)} = \bar{\mathbf{y}}^{(l-1)}$ is assumed to be sparse. The output of $l$'th block $\bar{\mathbf{y}}^{(l)} = \mathbf{g}(\mathbf{V}_{n^{(l)}}\mathbf{z}^{(l)}) = \mathbf{g}(\mathbf{V}_{n^{(l)}} \mathbf{W}^{(l)}\mathbf{q}^{(l)}) $ is also sparse, and this output is used as the input to the succeeding $(l+1)$'th block.  This means $\bar{\mathbf{y}}^{(l)} = \mathbf{q}^{(l+1)}$. Then, the output of $(l+1)$'th block is
\begin{eqnarray*}
	\bar{\mathbf{y}}^{(l+1)}& =& \mathbf{g}(\mathbf{V}_{n^{(l+1)}}\mathbf{z}^{(l+1)}) = \mathbf{g}(\mathbf{V}_{n^{(l+1)}} \mathbf{W}^{(l+1)}\mathbf{q}^{(l+1)})\nonumber\\& =& \mathbf{g}(\mathbf{V}_{n^{(l+1)}} \mathbf{W}^{(l+1)}\bar{\mathbf{y}}^{(l)}) \nonumber\\&=& \mathbf{g}(\mathbf{V}_{n^{(l+1)}} \mathbf{W}^{(l+1)} \mathbf{g}(\mathbf{V}_{n^{(l)}} \mathbf{W}^{(l)}\mathbf{q}^{(l)}))
\end{eqnarray*}
Corresponding to the two vectors $\mathbf{q}^{(l)}_1$ and $\mathbf{q}^{(l)}_2$, and their appropriate transforms, we have the following relations
\begin{align*}
	&\kappa_l(1-\delta_l) \| \mathbf{q}_1^{(l)} - \mathbf{q}_2^{(l)} \|^2 \nonumber\\&\quad\leq \| \bar{\mathbf{y}}_1^{(l)} - \bar{\mathbf{y}}_2^{(l)} \|^2 
	\leq (1+\delta_l) \| \mathbf{q}_1^{(l)} - \mathbf{q}_2^{(l)} \|^2, \\
		&\kappa_{l+1}(1-\delta_{l+1}) \| \bar{\mathbf{y}}_1^{(l)} - \bar{\mathbf{y}}_2^{(l)} \|^2 \nonumber\\&\quad\leq \| \bar{\mathbf{y}}_1^{(l+1)} - \bar{\mathbf{y}}_2^{(l+1)} \|^2 \leq (1+\delta_{l+1}) \| \bar{\mathbf{y}}_1^{(l)} - \bar{\mathbf{y}}_2^{(l)} \|^2.
\end{align*}
As a consequence of the above relations, the feedforward chain with two blocks follows
\begin{align}
&	\kappa_{l}\kappa_{l+1}(1-\delta_l)(1-\delta_{l+1}) \| \mathbf{q}_1^{(l)} - \mathbf{q}_2^{(l)} \|^2 \nonumber\\&\quad\leq \| \bar{\mathbf{y}}_1^{(l+1)} - \bar{\mathbf{y}}_2^{(l+1)} \|^2 \leq (1+\delta_l) (1+\delta_{l+1}) \| \mathbf{q}_1^{(l)} - \mathbf{q}_2^{(l)} \|^2
	\label{eq:two_block_inequality}
\end{align}

\begin{mytheo}
	\label{theo:RobustnessOfANN_forNonOthogonalMatrix}
	Let a feedforward ANN using ReLU activation function be constructed as follows.
	\begin{enumerate}
		\item[(a)] The ANN comprises $L$ layers where the $l$'th layer has the transfer function $\bar{\mathbf{y}}^{(l)} = \mathbf{g}(\mathbf{V}_{n^{(l)}} \mathbf{W}^{(l)} \bar{\mathbf{y}}^{(l-1)})$. The $L$ blocks are in a chain. The input to the first block is $\mathbf{q}^{(1)} = \mathbf{x}$. The output of ANN is
		\begin{align*}
			\bar{\mathbf{y}}^{(L)} &\! = \mathbf{g}(\mathbf{V}_{n^{(L)}}\mathbf{z}^{(L)})  \nonumber\\& \!= \! \mathbf{g}(\mathbf{V}_{n^{(L)}} \! \mathbf{W}^{(L)} \! \mathbf{g}(\!\mathbf{V}_{n^{(L-1)}} \! \mathbf{W}^{(L-1)} \! \hdots \!\mathbf{g}(\mathbf{V}_{n^{(1)}} \! \mathbf{W}^{(1)} \mathbf{x}).
		\end{align*}
		\item[(b)] In the ANN, $\mathbf{W}^{(l)} \in \mathbb{R}^{n^{(l)} \times m^{(l)}}$ matrices are randomly constructed with appropriate sizes, that is $n^{(l)} = O( \nu^{(l)} \, \log m^{(l)})$ where $\nu^{(l)}$ is assumed a maximum sparsity level for $\mathbf{q}^{(l)}$, and $m^{(l)} = 2 n^{(l-1)} = 2 O( \nu^{(l-1)} \, \log m^{(l-1)})$.
	\end{enumerate}
	Then, the ANN provides both noise robustness and discriminative power properties jointly characterized by the following relation
	\begin{align}
		&\prod_{l=1}^L \kappa_{l}(1-\delta_l) \| \mathbf{x}_1 - \mathbf{x}_2 \|^2\nonumber\\& \quad\leq \| \bar{\mathbf{y}}_1^{(L)} - \bar{\mathbf{y}}_2^{(L)} \|^2 \leq \prod_{l=1}^L (1+\delta_l)   \| \mathbf{x}_1 - \mathbf{x}_2 \|^2,
	\end{align}
	where $\mathbf{x}_1$ and $\mathbf{x}_2$ are two input vectors to the ANN and their corresponding outputs are $\bar{\mathbf{y}}_1^{(L)}$ and $\bar{\mathbf{y}}_2^{(L)}$, respectively. We assume that $\mathbf{x}_1$ and $\mathbf{x}_2$ are also sparse in some basis.
\end{mytheo}
\begin{proof}
	The proof follows by applying the relation \eqref{eq:two_block_inequality} for all $l\in{1,L-1}$.
\end{proof}
\begin{mytheo}
	\label{theo:RobustnessOfANN_forOthogonalMatrix}
	 If we construct an ANN where $\mathbf{W}^{(l)} \in \mathbb{R}^{n^{(l)} \times m^{(l)}}$ matrices are randomly constructed othonormal matrices with $n^{(l)} \geq m^{(l)}$, then, the ANN will provide the following relation
	\begin{eqnarray}
		\prod_{l=1}^L \kappa_{l} \| \mathbf{x}_1 - \mathbf{x}_2 \|^2 \leq \| \bar{\mathbf{y}}_1^{(L)} - \bar{\mathbf{y}}_2^{(L)} \|^2 \leq  \| \mathbf{x}_1 - \mathbf{x}_2 \|^2.
	\end{eqnarray}
\end{mytheo}

\begin{myremark}[Tightness of bounds and effect of noise]
	\label{remark:noise_effect_in_block}
	We follow similar arguments in Remark~\ref{remark:noise_effect_in_q}.
	We interpret that a low perturbation noise passes through the block chain almost unhindered. On the other hand, a perturbation noise with high strength is attenuated. 
\end{myremark}

As an illustration of the concept, we show the plot showing the perturbation in output of one block alongwith the corresponding input perturbations in Figure 1.
In the experiment, we consider $\mathbf{q}_1$ to be a isotropic multivariate Gaussian $\mathcal{N}(\mathbf{0},\sigma^2\mathbf{I})$ with $\sigma=1$ with $m=16$ and $\mathbf{q}_2=\mathbf{q}_1+\Delta$, where $\Delta$ is drawn from isotropic multivariate Gaussian distribution $\mathcal{N}(\mathbf{0},\sigma^2\mathbf{I})$ with $\sigma=0.25$. We choose $\mathbf{W}$ with $m=n$ and $m=2n$ with entries drawn from $\mathcal{N}(0,\frac{1}{n})$ for various realizations of $\Delta$ and $\mathbf{q}_1$ and $\mathbf{q}_2$. The figure shows the scatter plot of $10^5$ samples. We observe that all values of  $\|\mathbf{y}_1-\mathbf{y}_2\|^2$ lie strictly below $\|\mathbf{y}_1-\mathbf{y}_2\|^2=\|\mathbf{q}_1-\mathbf{q}_2\|^2$ line. Further, we observe that the contraction is greater when a larger dimensional $\mathbf{W}$ is used. 
\begin{figure}
	\label{fig:scatter_plot}
	\centering
	\includegraphics[width=2.8in]{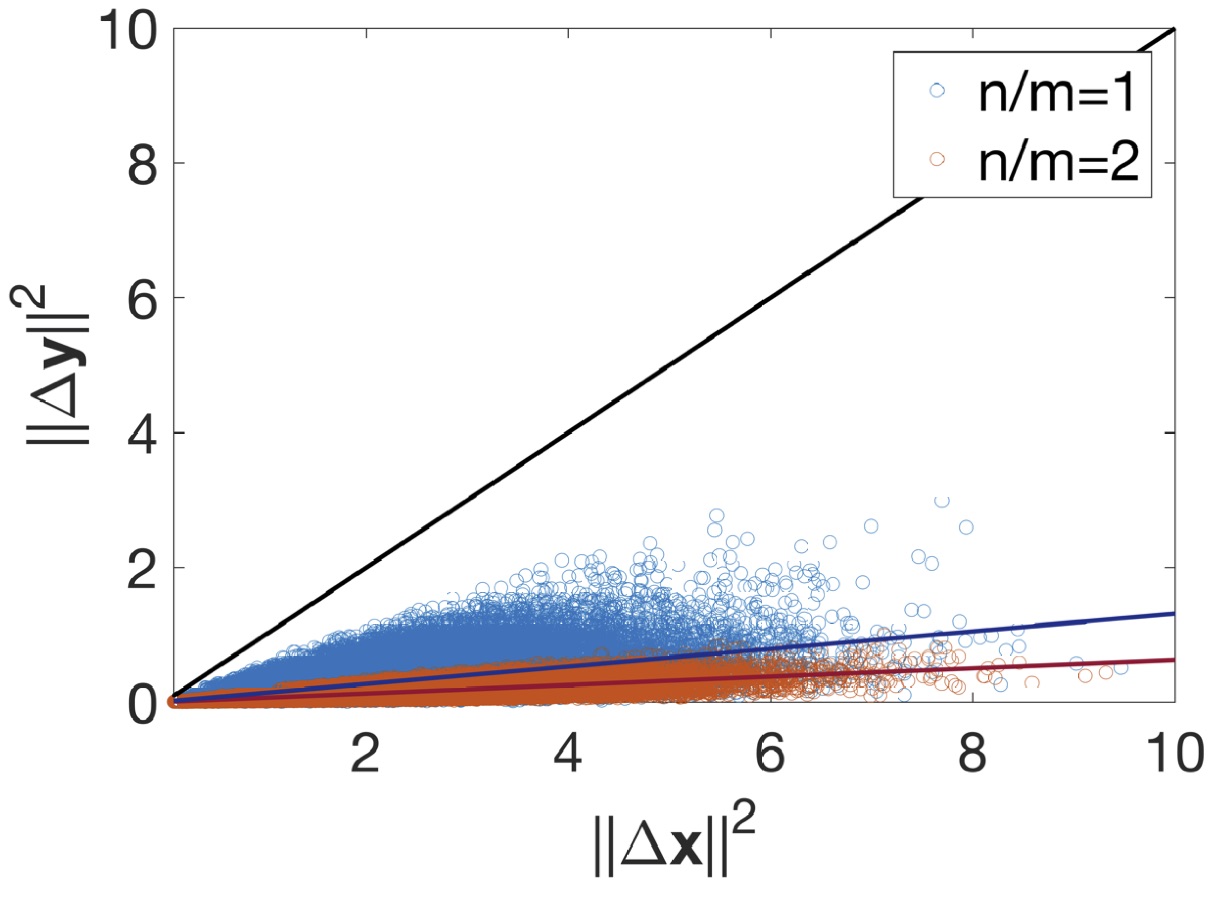}
	\caption{Noise robustness of the proposed ANN. The black line shows the reference line where the output perturbation is equal to the input perturbation. The thick blue and red lines indicate the least squares fit line to the datapoints of the respective colours.} 
	\vspace{-.2in}
\end{figure}
\section{Conclusion}
We show that random weights, sign splitter and rectified linear units provide a good combination to address two important properties of artificial neural networks--robustness and discriminative ability. We note the results with random orthonormal matrices are equally valid for standard real orthonormal matrices, for example, discrete cosine transform (DCT), Haar transform, Walsh-Hadamard transform, etc making our approach universal in nature. We believe that our analysis provides clues on the effectiveness of using random feature weights and ReLU functions in deep neural architectures.






\bibliographystyle{IEEEtran}
\bibliography{ref_r3net}
%


\end{document}